\documentclass[10pt,twocolumn,letterpaper]{article}
\usepackage{cvpr}
\usepackage{times}
\usepackage{epsfig}
\usepackage{graphicx}
\usepackage{amsthm}
\usepackage{amsmath}
\usepackage{amssymb}
\usepackage{subfigure}
\usepackage{float}
\usepackage{algorithm}
\usepackage{algorithmic}
\usepackage{comment}

\usepackage[normalem]{ulem}
\usepackage{color}

\newtheorem{theorem}{Theorem}
\newtheorem{lemma}[theorem]{Lemma}

\usepackage[breaklinks=true,bookmarks=false]{hyperref}

\cvprfinalcopy 

\def\cvprPaperID{****} 

\begin{document}

\title{Matrix Variate RBM and Its Applications}

\author{Guanglei Qi$^1$, Yanfeng Sun$^1$, Junbin Gao$^2$, Yongli Hu$^1$ and Jinghua Li$^1$\\
{\small $^1$College of Metropolitan Transportation, Beijing University of Technology, Beijing 100124, China}\\
{\small $^2$Business Analytics Discipline, The University of Sydney Business School, Camperdown NSW 2006, Australia}\\
{\tt\small qgl@emails.bjut.edu.cn; \{huyongli,yfsun,lijinghua\}@bjut.edu.cn; junbin.gao@sydney.edu.au}}

\maketitle

\begin{abstract}
Restricted Boltzmann Machine (RBM) is an important generative model modeling vectorial data. While applying an RBM in practice to images, the data have to be vectorized. This results in high-dimensional data and valuable spatial information has got lost in vectorization. In this paper, a  Matrix-Variate Restricted Boltzmann Machine (MVRBM) model is proposed by generalizing the classic RBM to explicitly model matrix data. In the new RBM model, both input and hidden variables are in matrix forms which are connected by bilinear transforms. The MVRBM has much less model parameters, resulting in a faster training algorithm while retaining comparable performance as the classic RBM. The advantages of the MVRBM have been demonstrated on two real-world applications: Image super-resolution and handwritten digit recognition.

\end{abstract}

\section{Introduction}\label{Sec:1}
A Boltzmann machine as a type of stochastic recurrent neural network was invented by Hinton and Sejnowski in 1985~\cite{HintonSejnowski1983}. However it is not efficient to use the generic Boltzmann machines in machine learning or inference due to its unconstrained connectivity among variable units. To make a practical model, Hinton~\cite{Hinton2002} proposes an architecture called the \textit{Restricted Boltzmann Machine} (RBM), only units between visible layer and hidden layer connected.

With the restricted connectivity between visible and hidden units, an RBM can be regarded as a probabilistic graphical model with bipartite graph structure. In recent years, RBMs have attracted considerable research interest in pattern recognition~\cite{Bishop2006,SocherChenManningAndrew2013} and machine learning~\cite{Bengio2009,HintonSalakhutdinov2009,KrizhevskyHinton2010,MemisevicHinton2010,TielemanHinton2009}, due to their strong ability in feature extraction and representation.

Units at visible and hidden layers are connected through the restricted linear mapping with weights to be trained. Given some training data, the goal of training a RBM model is to learn the weights between visible and hidden units such that the probability distribution represented by a RBM fits the training samples as well as possible. A well trained RBM can provide efficient representation for new input data following the  same distribution as training data.

The classic RBM model is mainly designed for vectorial input data or variables. However, data emerging from modern science and technology are in more general structures. For example, digital images are collected as 2D matrices, which reflect the spatial correlation or information among pixels.  In order to apply the classic RBM to such 2D image data, a typical workaround  is to vectorize 2D data. Unfortunately such as a vectorization process not only breaks the inherent high-order image structure, resulting in losing important information about interaction across modes, but also leads to increasing the number of model parameters induced by a full connection between visible and hidden units.

To extend the classic RBM for 2D matrix data, in this paper, we propose a Matrix-Variate Restricted Boltzmann Machine (MVRBM) model. Like the classic RBM, the MVRBM model also defines a probabilistic model for binary units arranged in a bipartite graph, but topologically units on the same layer (input or hidden) are organized in 2D arrays and connected through a bilinear mapping, see Section~\ref{Sec:4}. In fact, the proposed bilinear mapping specifies a specific structure in the parameters of the model, thus gives raise to reduce the number of parameters to be learned in training process.

In summary, the new model has the following advantages which make up our contributions in this paper:
\begin{enumerate}
\item The total number of parameters to be learned is significantly less than that in the traditional RBMs, thus the computational complexity in training and inferring can be significantly improved.

\item Both the visible layer and hidden layer are organized in the matrix format, thus the spatial information in 2D matrix data can be maintained in the training and inference processes and better performance in reconstruction can be achieved.

\item The idea presented in MVRBM can be easily extended to any order tensorial data, thus the basic RBM can be applied to more complex data structures.
\end{enumerate}

The rest of the paper is organized as follows. In Section \ref{Sec:2}, we summarize the related works to further highlight our contributions.  In Section \ref{Sec:4}, the MVRBM model is introduced and a stochastic learning algorithm based on \textit{Contrast Divergence} (CD) is proposed. In Section \ref{Sec:5}, the performance of the proposed method is evaluated on two computer vision tasks handwritten digit recognition and image super-resolution. Finally, conclusions and suggestions for future work are provided in Section~\ref{Sec:6}.

\section{Related Works}\label{Sec:2}
There have been more and more multiway data acquired in modern scientific and engineering research, e.g., medical images~\cite{AdaliLevinCalhoun2015,LuHaligWangChenFei2014}, multispectral images~\cite{BernabeMarpuPlazaMuraBenediktsson2014,Garzelli2015}, and video clips~\cite{GuyByrneRich2014} etc. 
It is well known that vectorizing multiway data results in correlation information loss, thus downgrade the performance of learning algorithm for vectorial data like the classic RBMs. In recent years, research works on learning algorithms for multiway data modeling have attracted great attention.

Rovid et al.~\cite{RovidSzeidlVarlaki2011} propose a tensor-product model based representation of neural networks in which a neural network structure for mutual interaction of variable components is introduced. It conceptually restructures the tensor product model transformation to a generalized form for fuzzy modeling and to propose new features and several new variants of the tensor product model transformation. However this type of neural networks is actually defined for vectorial data rather than tensorial variates. The similar idea can be seen in the most recent paper~\cite{HutchinsonDengYu2013}. The key characterization for these networks is the connection weights (neural networks parameters) are in tensor format, rather than the data variables in the networks. Thus except for the nonlinearity introduced by the activation function, the neural networks offer the capacity of encoding nonlinear interaction among the hidden variable components. The so-called tensor analyzer~\cite{TangSalakhutdinovHinton2013} also serves as such an example.

Socher et al.~\cite{SocherChenManningAndrew2013} present another similar work. It uses the similar structure as proposed in~\cite{HutchinsonDengYu2013} to generalize several previous neural network models and provide a more powerful way to model correlation information than a standard neural network layer.

There are several works on multiple ways Boltzmann machine~\cite{TaylorHinton2009,ZhaoAmmarRoos2013}. Taylor and  Hinton~\cite{TaylorHinton2009} propose a factored conditional restricted Boltzmann Machines for modeling motion style. In order to capture context of motion style, this model takes  history and current information as input data, thus connections from the past to current visible units and the hidden units are increased. In this model, the input data consist of two vectors, and the output data is also in vector form and the weights between visible and hidden units are matrix.

Zhao et al.~\cite{ZhaoAmmarRoos2013} use some video sequences for training RBM to get better classification performance. The video sequences are also vectorized as some vectors, used as the input to a classic RBM with a connection defined by a tensor weight.

All the above attempts aim to model the interaction between the components of vector variates. Similar to the classic RBM, these models are not appropriate for the matrix inputs. To the best of our knowledge, the first work aiming at modeling matrix variate inputs is proposed by Nguyen et al.~\cite{NguyenTranPhungVenkatesh2015}, named \textit{Tensor-variate Restricted Boltzmann Machines} (TvRBM) model. The authors have demonstrated the capacity of the model on three real-world applications with convincible performance.
In its model architecture, the input is designed for tensorial data including matrix data while the hidden units are organized as a vector. The connection between the hidden layer and the visible layer is defined by the linear combination over tensorial weights. To reduce the number of weight parameters, the weight tensors are further specified as the so-called rank-r tensors~\cite{KoldaBader2009}. However our criticism over TvRBM is that the specification of the rank-r tensor weights is too restrictive to efficiently empower the model capability.

Another model related to our proposed MVRBM is the so-called \textit{Replicated Softmax RBM} (RS-RBM)~\cite{SalakhutdinovHinton2010}. Similar to TvRBM in~\cite{NguyenTranPhungVenkatesh2015}, RS-RBM uses a linear mapping between a matrix input layer and a hidden vector layer. To model document topics in terms of word counts, an implicit condition is imposed on the matrix input, i.e., the sum of the binary entries of each row in the matrix input must be 1. Thus the Replicated Softmax model is actually equivalent to an RBM of vector softmax input units with identical weights for each unit.

Our proposed MVRBM in the next section is different from both TvRBM and RS-RBM in several aspects. First, the binary entries in matrix input for MVRBM are independent as that in TvRBM while they are dependent in RS-RBM. Second, the hidden layer units in MVRBM are in a matrix format rather than in a vector format as in both RS-RBM and TvRBM. Third, the linear mapping between input and hidden layers in MVRBM is bilinear.

\section{Matrix Variate Restricted Boltzmann Machines (MVRBM)}\label{Sec:4}
In this section, we will present the proposed MVRBM and investigate its learning algorithm.

\subsection{Model Definition}

The classic RBM~\cite{FischerIgel2012,HintonSalakhutdinov2006} is a bipartite undirected probabilistic graphical model with stochastic visible units $\mathbf x$ and stochastic hidden units $\mathbf y$, both are in vector. The model is shown in Figure~\ref{figure1} where each visible unit (represented by a cubic) is connected to each hidden unit (represented by a cylinder).
\begin{figure}[H]
\centering
\includegraphics[width=0.3\textwidth]{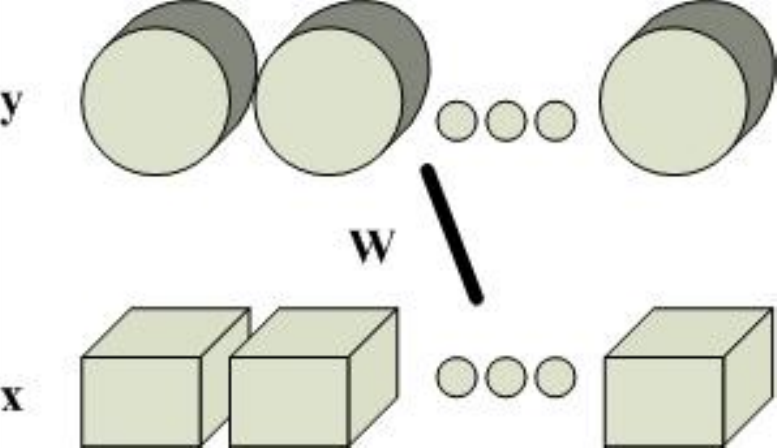}
\caption{Graphical Illustration of RBM.}
\label{figure1}
\end{figure}

The RBM assigns energy for a joint configuration $(\mathbf x,\mathbf y)$:
\begin{equation}
E(\mathbf x,\mathbf y;\Theta ) =  -\mathbf x^T W\mathbf y - \mathbf b^T\mathbf x - \mathbf c^T \mathbf y, \label{Energy1}
\end{equation}
where $\mathbf x \in\mathbb{R}^I$, $\mathbf y\in\mathbb{R}^K$ are the binary states of visible units   and hidden units, $\mathbf b\in \mathbb{R}^I$ and $\mathbf c\in\mathbb{R}^K$ are the biases, and $W \in {\mathbb{R}^{I \times K}}$ represents visible-to-hidden symmetric interaction terms in the neural network. Denote by $\Theta  = \left\{ \mathbf b,\mathbf c, W \right\}$ all the model parameters.

To introduce our proposed MVRBM, we define the following notations. Denote by $X = [x_{ij}]\in\mathbb{R}^{I\times J}$ the binary visible matrix variate, and $Y=[y_{kl}]\in\mathbb{R}^{K\times L}$ the binary hidden matrix variate. We assume that the independent random variables $x_{ij}$ and $y_{kl}$ all take values from $\{0,\ 1\}$. Given the parameters of a 4-order tensor $\mathcal{W}=[w_{ijkl}]\in\mathbb{R}^{I\times J\times K\times L}$, bias matrices $B=[b_{ij}]\in\mathbb{R}^{I\times J}$ and $C=[c_{kl}]\in\mathbb{R}^{K\times L}$, we define the following energy function for a joint configuration $(X, Y)$,
\begin{equation}
\begin{aligned}
 E(X, Y; \Theta) &= \sum\limits_{i = 1}^I {\sum\limits_{j = 1}^J {\sum\limits_{k = 1}^K {\sum\limits_{l = 1}^L {{x_{ij}}{w_{ijkl}}{y_{kl}}} } } }  \\
        &+ \sum\limits_{i = 1}^I {\sum\limits_{j = 1}^J {{x_{ij}}{b_{ij}}} }  + \sum\limits_{k = 1}^K {\sum\limits_{l = 1}^L {{y_{kl}}{c_{kl}}}},
\end{aligned}\label{equation3}
\end{equation}
where $\Theta = \{\mathcal{W}, B, C\}$ collects all the parameters.

There are a total number of $I\times J\times K\times L + I\times J+ K\times L$ free parameters in $\Theta$. This is a huge number even for mild values of $I$, $J$, $K$ and $L$ and requires a large amount of training samples and times. In order to reduce the number of free parameters to save computational complexity in training and inference, we intend to specify a multiplicative interaction between visible units and hidden units by taking $w_{ijkl}=u_{ki} v_{lj}$. By defining two new matrices $U=[u_{ki}]\in\mathbb{R}^{K\times I}$ and $V=[v_{lj}]\in\mathbb{R}^{L\times J}$, we can re-write the
energy function \eqref{equation3} into the following form,
\begin{equation}
 E(X, Y; \Theta)
  =  - \text{tr}(U^T Y V X^T) - \text{tr}(X^TB) - \text{tr}(Y^TC). \label{equation4}
\end{equation}
Both matrices $U$ and $V$ jointly define the interaction between input matrix $X$ and hidden matrix $Y$. The total number of free parameters in \eqref{equation3} has been reduced to  $I\times K+ L\times J + I\times J + K\times L$ in \eqref{equation4}.

Based on \eqref{equation4}, we define the following distribution:
\begin{equation}
p(X, Y; \Theta) = \frac{1}{{Z(\Theta )}}\exp \left\{ { - E(X, Y; \Theta)} \right\},\label{equation5}
\end{equation}
where $\Theta$ denotes all the model parameters $U$, $V$, $B$ and $C$ and the normalization constant $ Z(\Theta ) $ is defined by
\begin{equation}
Z(\Theta ) = \sum\limits_{X \in \mathcal{X}, Y \in \mathcal{Y}} {\exp \left\{ { - E(X, Y; \Theta)} \right\}},
\end{equation}
where $\mathcal{X}$ and $\mathcal{Y}$  are the binary value spaces of $X$ and $Y$, respectively.

We call the probabilistic model defined by \eqref{equation5} the Matrix Variate RBM (MVRBM). The model is shown in Figure~\ref{Figure2}.
\begin{figure}[H]
\centering
\includegraphics[width=0.48\textwidth]{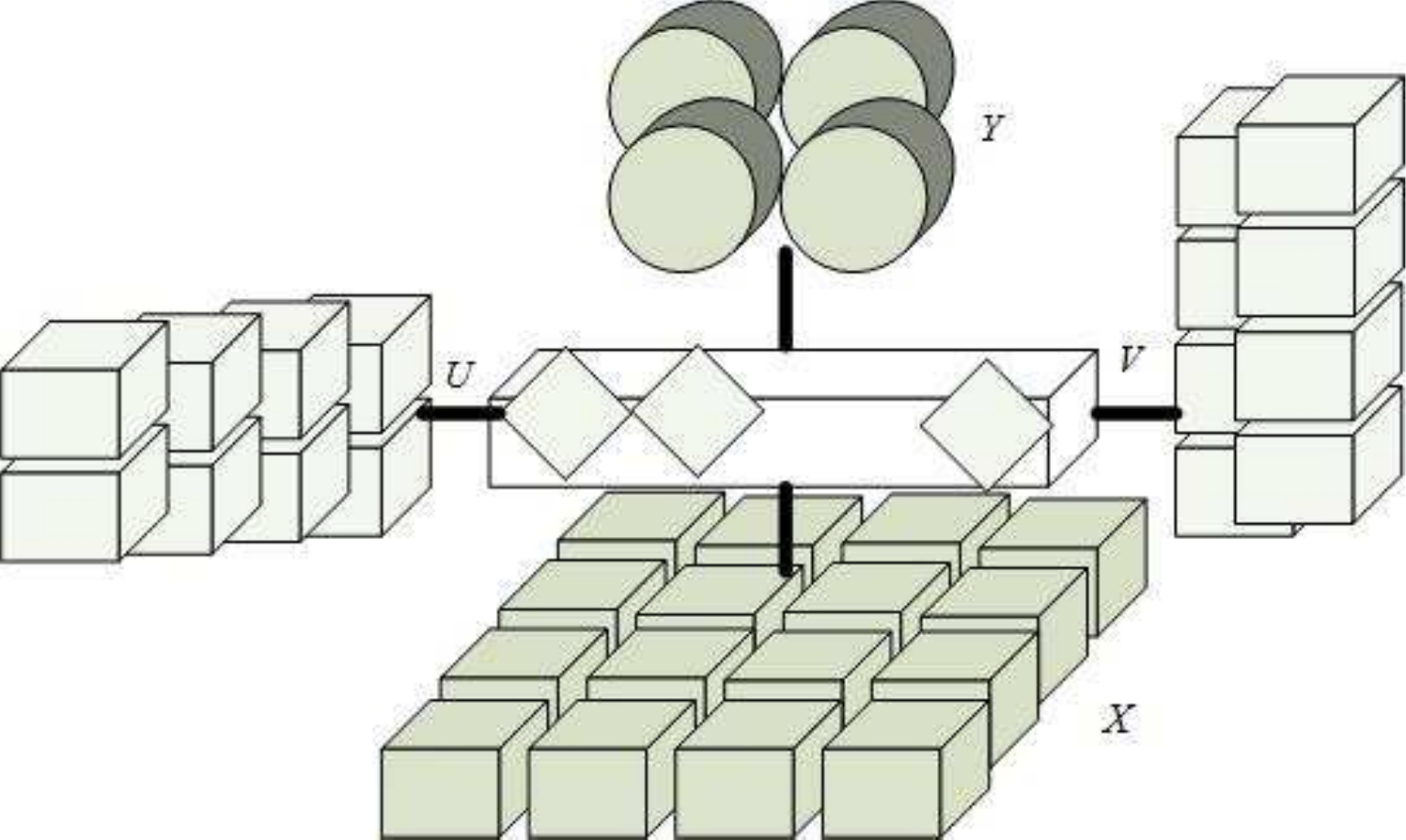}
q\caption{Graphical Illustration of MVRBM}
\label{Figure2}
\end{figure}

To facilitate exploring learning algorithm for MVRBM, we propose the following lemma regarding the conditional density of each visible and hidden entry.
\begin{lemma}Let the MVRBM model be defined by \eqref{equation4} and \eqref{equation5}, the conditional density of each visible entry $x_{ij}$ over all the other variables is given by
\begin{equation}
p({x_{ij}} = 1\left|Y; \Theta \right.) = \sigma ({b_{ij}} + \sum\limits_{k = 1}^K {\sum\limits_{l = 1}^L {{y_{kl}}{u_{ki}}{v_{lj}}} } ) \label{equation9}
\end{equation}
and the conditional density of each hidden entry $y_{kl}$ over all the other variables is given by
\begin{equation}
p({y_{kl}} = 1\left|X; \Theta \right.) = \sigma ({c_{kl}} + \sum\limits_{i = 1}^I {\sum\limits_{j = 1}^J {{x_{ij}}{u_{ki}}{v_{lj}}} } ),\label{equation9a}
\end{equation}
where $\sigma$ is the sigmoid function $\sigma (x) = 1/(1 + {e^{ - x}}).$
\end{lemma}
\begin{proof} Both $X$ and $Y$ are in symmetric position, so we only prove \eqref{equation9} as an example. For this purpose, denote
\[
m_{ij} = b_{ij} + \sum_{k = 1}^K\sum_{l = 1}^L y_{kl}u_{ki} v_{lj}.
\]
Consider a fixed entry $x_{i_0j_0}$.
First by Bayes theorem we have
\[
p(x_{i_0j_0}=1|Y; \Theta) = \frac{p(x_{i_0j_0}=1, Y; \Theta)}{p(Y; \Theta)}.
\]
Clearly
\begin{align*}
p(Y; \Theta) & = \frac1{Z(\Theta)}\sum_{X\in\mathcal{X}}\exp\{-E(X,Y; \Theta)\} \\
&=\frac{\exp(\text{tr}(Y^TV))}{Z(\Theta)}\sum_{X\in\mathcal{X}}\prod_{(i,j)} \exp(m_{ij}x_{ij}) \\
&=\frac{\exp(\text{tr}(Y^TV))}{Z(\Theta)} \prod_{(i,j)} \left(1 + \exp(m_{ij})\right).
\end{align*}
Similarly, denoting by $X_{/i_0j_0}$ all the entries except for $x_{i_0j_0}$, we have
\begin{align*}
& p(x_{i_0j_0}=1, Y; \Theta) \\
= & \frac{\exp(\text{tr}(Y^TV))}{Z(\Theta)}\exp(m_{i_0j_0}) \sum_{X_{/i_0j_0} \in\mathcal{X}} \prod_{(i,j)\not=(i_0j_0)} \exp ( m_{ij}x_{ij} )\\
=&\frac{\exp(\text{tr}(Y^TV))}{Z(\Theta)}\exp(m_{i_0j_0}) \prod_{(i,j)\not=(i_0,j_0)} \left(1 + \exp(m_{i_0j})\right).
\end{align*}
Combining these together gives
\[
p(x_{i_0j_0}=1|Y; \Theta) = \frac{\exp(m_{i_0j_0})}{1 +\exp(m_{i_0j_0}) } = \sigma(m_{i_0j_0}),
\]
which is \eqref{equation9} for $(i,j) = (i_0,j_0)$. This completes the proof.
\end{proof}

\textit{Remark 1:} In terms of matrix representation, the two conditional probabilities can be written as
\begin{align}
&p(X = 1 | Y; \Theta) = \sigma(U^T Y V + B), \label{ConditionX}\\
&p(Y = 1 | X; \Theta) = \sigma(UXV^T + C),\label{ConditionY}
\end{align}
where the sigmoid function $\sigma$ applies on the entries of the corresponding matrices.

\subsection{The Maximum Likelihood and CD Algorithm for MVRBM}

Let $\mathcal{D} = \{X_1, ..., X_N\} $ be an observed dataset. Under the joint distribution \eqref{equation5}, the log likelihood of  $\mathcal{D}$ is defined by
\[\ell  = \frac{1}{N}\sum\limits_{n = 1}^N {\log (\sum\limits_{Y\in \mathcal{Y}} {\exp \left\{ { - E(X_n, Y)} \right\}} )}  - \log Z(\Theta ).\]

For any component $\theta $ of $\Theta $ , we can prove that
\begin{equation}
\begin{aligned}
\frac{{\partial \ell }}{{\partial \theta }} =& - \frac{1}{N}\sum\limits_{n = 1}^N {\sum\limits_{Y \in \mathcal{Y}} {p(Y\left| {X_n;\Theta} \right.)} \frac{{\partial E(X_n, Y; \Theta)}}{{\partial \theta }}}  \\
&+ \sum\limits_{X' \in \mathcal{X}, Y' \in \mathcal{Y}} {p(X', Y'; \Theta)\frac{\partial E(X', Y'; \Theta)}{\partial \theta}}.
\end{aligned}\label{equation7}
\end{equation}
We call the first term of the right hand side of \eqref{equation7} the data expectation and the second term the model expectation.

The main difficulty in calculating the derivative of the likelihood function with respect to a parameter is to fast compute the model expectation. The model expectation is intractable due to the summation over all the possible visible and hidden states. However, the \textit{Contrast Divergence} (CD) procedure allows fast approximation using short Markov chains. The main idea in the CD algorithm is as follow: a Gibbs chain is initialized with a training example $X^{(0)}_n = X_n$ of the training set, then alternatively using \eqref{ConditionX} and \eqref{ConditionY} gives the chain $\{(X^{(0)}_n, Y^{(0)}_n), (X^{(1)}_n, Y^{(1)}_n), ..., (X^{(k)}_n, Y^{(k)}_n), ...\}$. The CD-k algorithm takes the samples $\{X^{(k)}_n\}^N_{n=1}$ at step $k$ to approximate the model expectation, that is,
\begin{align}
&\sum_{X' \in \mathcal{X}, Y' \in \mathcal{Y}} p(X', Y'; \Theta)\frac{\partial E(X', Y'; \Theta)}{\partial \theta} \notag\\
=& \sum_{X'\in\mathcal{X}}\left(\sum_{Y'\in\mathcal{Y}} p(Y' | X';\Theta)\frac{\partial E(X', Y'; \Theta)}{\partial \theta}\right)p(X';\Theta) \notag\\
\approx & \frac1N\sum^N_{n=1}\sum_{Y'\in\mathcal{Y}}p(Y' | X^{(k)}_n;\Theta)\frac{\partial E(X^{(k)}_n, Y'; \Theta)}{\partial \theta}.\label{ApproximatedExpectation}
\end{align}
\eqref{ApproximatedExpectation} is actually the data expectation over the sampled data in the $k$-th step of all the Gibbs chains from training samples, which is similar to the data expectation in \eqref{equation7}. Finally the CD algorithm is implemented by
\begin{align}
\frac{\partial \ell}{\partial \theta}\approx & -\frac1N\sum^N_{n=1}\sum_{Y\in\mathcal{Y}}p(Y | X_n;\Theta)\frac{\partial E(X_n, Y; \Theta)}{\partial \theta} \label{ApproxLikelihood}\\
& + \frac1N\sum^N_{n=1}\sum_{Y'\in\mathcal{Y}}p(Y' | X^{(k)}_n;\Theta)\frac{\partial E(X^{(k)}_n, Y'; \Theta)}{\partial \theta}.\notag
\end{align}

For all the four parameters in the MVRBM, we take calculating $\frac{\partial \ell}{\partial U}$ as an example. The derivatives with respect to other parameters can be calculated similarly. From \eqref{equation4}, we have
\[
\frac{\partial E(X,Y;\Theta)}{\partial U} =  - YVX^T.
\]
In this case,  the derivative \eqref{ApproxLikelihood} becomes
\begin{align}
 \frac{\partial \ell}{\partial U} \approx & - \frac{1}{N}\sum\limits_{n = 1}^N \left(\sum\limits_{Y \in \mathcal{Y}} p(Y\left| X_n;\Theta\right.) \right) V X_n^T  \label{equation10}\\
&+\frac1N\sum^N_{n=1} \left(\sum\limits_{ Y' \in \mathcal{Y}} p(Y'\left| X^{(k)}_n;\Theta \right.) Y'\right) V(X^{(k)}_n)^T.\notag
\end{align}

For binary variable $Y$ (or $Y'$ here), the mean value of $Y$ is equal to the probability of $Y=1$.
Hence the first term in \eqref{equation10} is, refer to \eqref{ConditionY},
\begin{equation*}
\sum\limits_{Y \in \mathcal{Y}} p(Y\left| X_n;\Theta\right.) Y = \sigma (U{X_n}{V^T} + C).
\end{equation*}
Similarly it applies to the second term. Thus we have
\begin{align}
\frac{\partial \ell }{\partial U} \approx & - \frac{1}{N}\sum\limits_{n = 1}^N\sigma(UX_nV^T + C) VX^T_n \notag\\
& + \frac1N \sum^N_{n=1} \sigma(UX^{(k)}_nV^T + C) V(X^{(k)}_n)^T.  \label{equation14}
\end{align}

Similarly for other parameters we obtain
\begin{align}
\frac{\partial \ell}{\partial V} \approx & - \frac{1}{N}\sum_{n = 1}^N \sigma(UX_nV^T+C)^TU{X}_n \notag\\
& + \frac{1}{N}\sum\limits_{n = 1}^N \sigma(UX^{(k)}_nV^T+C)^TUX^{(k)}_n, \label{equation15}\\
\frac{\partial \ell}{\partial B} \approx & - \frac{1}{N}\sum_{n = 1}^NX_n  + \frac{1}{N}\sum_{n = 1}^N X^{(k)}_n, \label{partialB}\\
\frac{\partial \ell}{\partial C} \approx & - \frac{1}{N}\sum_{n = 1}^N \sigma(UX_nV^T+C)  \notag \\
&+ \frac{1}{N}\sum_{n = 1}^N \sigma(UX^{(k)}_nV^T+C).
\label{equation17}
\end{align}

\textit{Remark 2:} As the model only depends on the product of parameters $U$ and $V$, one parameter may go up in any scale $s$ while the other goes down to $1/s$. To avoid the issue of un-identifying model parameters, we add a penalty of $\frac{\beta}2(\|U\|^2_F + \|V\|^2_F)$ to the log likelihood objective $\ell$.

We summarize the overall CD procedure for Matrix Variate RBM in Algorithm~\ref{Algorithm1}.
In all our experiments, we use the special CD-$1$ algorithm for training.
\begin{algorithm}
\renewcommand{\algorithmicrequire}{\textbf{Input:}}
\renewcommand\algorithmicensure {\textbf{Output:} }
\caption{CD-$K$ algorithm for MVRBM:}\label{Algorithm1}
\begin{algorithmic}[1]
\REQUIRE A set of training data of $N$ matrices $\mathcal{D} = \{X_1, ..., X_N\}$, the maximum iteration number $T$ (default value $=10,000$), the learning rate $\alpha$ (default value $= 0.05$), the weight regularizer $\beta$ (default value $= 0.01$), the momentum $\gamma$ (default value $= 0.5$), the batch size $b$ (default value $= 100$) and the CD step $K$ (default value $=1$).
\ENSURE  Model parameters  $\Theta=\{U,V,B,C\}$.
\STATE   \textbf{Initialization}: Randomly initialize values for $U$ and $V$, set the bias $B=0$ and $C=0$ and the gradient increments $\Delta U = \Delta V = \Delta B= \Delta C = 0$.
\FOR{iteration step $t=1\rightarrow T$}
\STATE Randomly divide $\mathcal{D}$ into $M$ batches $\mathcal{D}_1, ..., \mathcal{D}_M$ of size $b$, then
\FOR{batch $m=1\rightarrow M$}
\STATE For all the data $X^{(0)}=X \in \mathcal{D}_m$ run the Gibbs sampling at the current model parameters $\Theta$:
\FOR{$k=0\to K-1$}
    \STATE sample  $Y^{(k)}$ according to \eqref{ConditionY} with the current $X^{(k)}$;
    \STATE sample  $X^{(k+1)}$ according to \eqref{ConditionX} with $Y^{(k)}$;
\ENDFOR
\STATE Update the gradient increment with $\mathcal{D}_m$ and $\mathcal{D}^{(K)}_m (X^{(K)})$ by using \eqref{equation14} to \eqref{equation17}:
\begin{align*}
\Delta U &= \gamma \Delta U + \alpha \left(-\left.\frac{\partial \ell}{\partial U}\right|_{\mathcal{D}_m,\mathcal{D}^{(K)}_m} - \beta U\right); \\
\Delta V &= \gamma \Delta V + \alpha \left(- \left.\frac{\partial \ell}{\partial V}\right|_{\mathcal{D}_m,\mathcal{D}^{(K)}_m} - \beta V\right); \\
\Delta B  &= \gamma \Delta B  + \alpha \left(- \left.\frac{\partial \ell}{\partial B}\right|_{\mathcal{D}_m,\mathcal{D}^{(K)}_m}\right); \\
\Delta C  &= \gamma  \Delta C   + \alpha \left(- \left.\frac{\partial \ell}{\partial C}\right|_{\mathcal{D}_m,\mathcal{D}^{(K)}_m}\right);
\end{align*}
\STATE Update model parameters $\theta \in \Theta$ with
\[
\theta \leftarrow \theta + \Delta\theta;
\]
\ENDFOR
\ENDFOR
\end{algorithmic}
\end{algorithm}

\textit{Remark 3:} The MVRBM model can be easily extended to any order tensorial input and hidden units. Note that the energy \eqref{equation4}, which is equivalent to
\[
E(X,Y;\Theta) = - \langle Y, UXV^T\rangle - \langle B, X\rangle - \langle C, Y\rangle,
\]
determines the bilinear mappings between $X$ and $Y$
\[
Y = X\times_1 U\times_2 V + E_y
\text{ and }
X = Y\times_1 U^T\times_2 V^T + E_x,
\]
where $\times_n$ means $n$-mode product of a tensor and a matrix~\cite{KoldaBader2009} and $E$'s are the logistic error matrices. For $D$-order tensorial binary variates $\mathcal{X}$ and $\mathcal{Y}$, a Tucker decomposition~\cite{KoldaBader2009}
$
\mathcal{Y} = \mathcal{X}\times_1 U_1\times_2 \cdots \times_D U_D + \mathcal{E}
$
suggests the following energy function
\[
E(\mathcal{X}, \mathcal{Y}; \Theta) = -\langle \mathcal{Y}, \mathcal{X}\times_1 U_1\times_2 \cdots \times_D U_D\rangle - \langle \mathcal{B}, \mathcal{X}\rangle - \langle \mathcal{C}, \mathcal{Y}\rangle
\]
for a Tensorial Variate RBM (TV-RBM). Algorithm and its applications of such a TV-RBM will be investigated in our forthcoming paper~\cite{QiSunGaoHu2015}.

\subsection{Multimodal MVRBM}
Information in the real world comes through multiple input channels. For example, in image super-resolution, the lower resolution images are associated with different types of features. The classic RBM has been engineered to handle multimodal data~\cite{SalakhutdinovHinton2010}.

As a proof of the concept, in this subsection, we simply give an outline to describe how the newly proposed MVRBM can be generalized to process multimodal matrix variates. We assume that the visible layer consists of two separate matrices $X\in\mathbb{R}^{I\times J}$ and $Z\in\mathbb{R}^{H\times W}$. Both $X$ and $Z$ could be in the same dimension and will be connected to the hidden layer given by a matrix variate $Y\in\mathbb{R}^{K\times L}$. The connection is specified in the following energy function
\begin{align}
 E(X, Z, Y; \Theta) =&  - \text{tr}(U^TYVX^T) - \text{tr}(X^TB) - \text{tr}(Y^TV)  \notag\\
&- \text{tr}(Q^TYRZ^T) - \text{tr}(Z^TA). \label{EnergyMultiview}
\end{align}

Given the energy function in \eqref{EnergyMultiview}, the following joint distribution
\[
p(X,Z,Y;\Theta) = \frac1{Z(\Theta)} \exp \{-E(X, Z, Y; \Theta)\}
\]
defines a graphical model, called \textit{Multimodal} MVRBM (MMVRBM). The learning algorithm based on CD approximation can be easily derived. In experiment, we will use this model for image super-resolution.

\section{Experimental Results and Analysis}\label{Sec:5}
In this section, we implement Algorithm 1 and conduct a number of experiments on some  public databases to assess the proposed MVRBM model. These experiments are designed to demonstrate the performance of MVRBM in feature extraction and reconstruction by comparing with the existing RBM model. The algorithm is coded by Matlab, and is run on a machine with a 2.50GHz Intel Xeon Processor and 128GB of installed memory.

\subsection{Experiment 1: Denoising and Reconstruction}
In the first experiment, our goal is to show that a well trained MVRBM can be used to data denoising and dimension reduction for reconstruction. The experiment is based on the MNIST handwritten digit database, downloadable from  \texttt{http://yann.lecun.com/exdb/mnist/.} The dataset contains 70000 labeled images of handwritten digits which are in 28 by 28 pixels and gray scale. In general, the training set contains 60,000 images and the testing set 10,000 images.

In the first attempt, we wish to demonstrate that the MVRBM model actually learns the information from data. For this purpose, we randomly select 5,000 images of digit 9 from the training dataset, set the hidden matrix variate to size $15\times 15$ and use most of default parameter values set in Algorithm~\ref{Algorithm1}. The training process terminates after $T = 3,000$ epoches.
With the trained MVRBM, we conduct a simple denoising test. We randomly add 10\% salt \& pepper  noises to some testing images of digit 9. The noised testing images are shown in Figure~\ref{Figure1Gao}(a) while the denoised counterparts in Figure~\ref{Figure1Gao}(b). The denoising result is visually pleasing.

\begin{figure}
\centering
\subfigure[noised digits]{\includegraphics[width=0.22\textwidth]{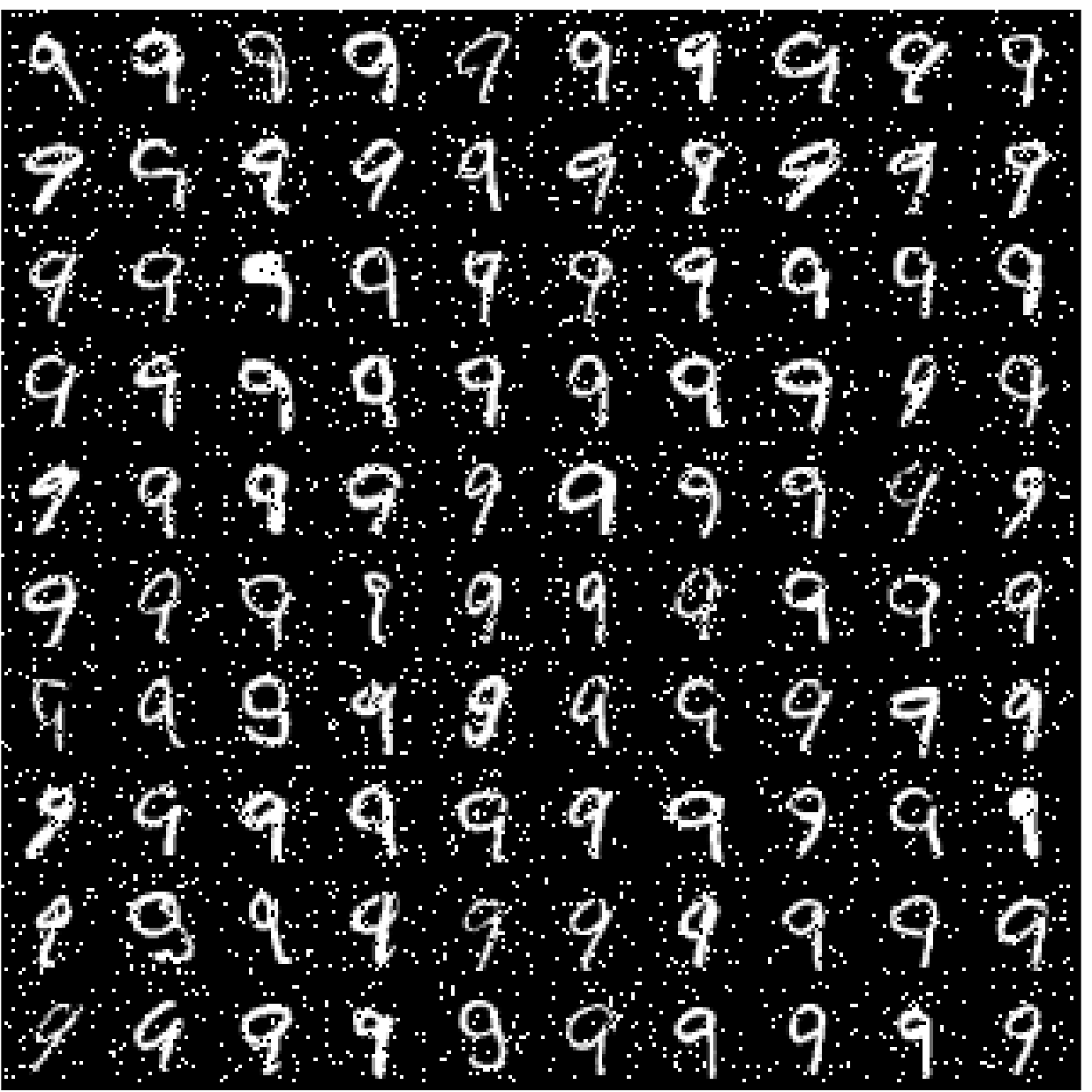}}\;\;
\subfigure[denoised digits]{\includegraphics[width=0.22\textwidth]{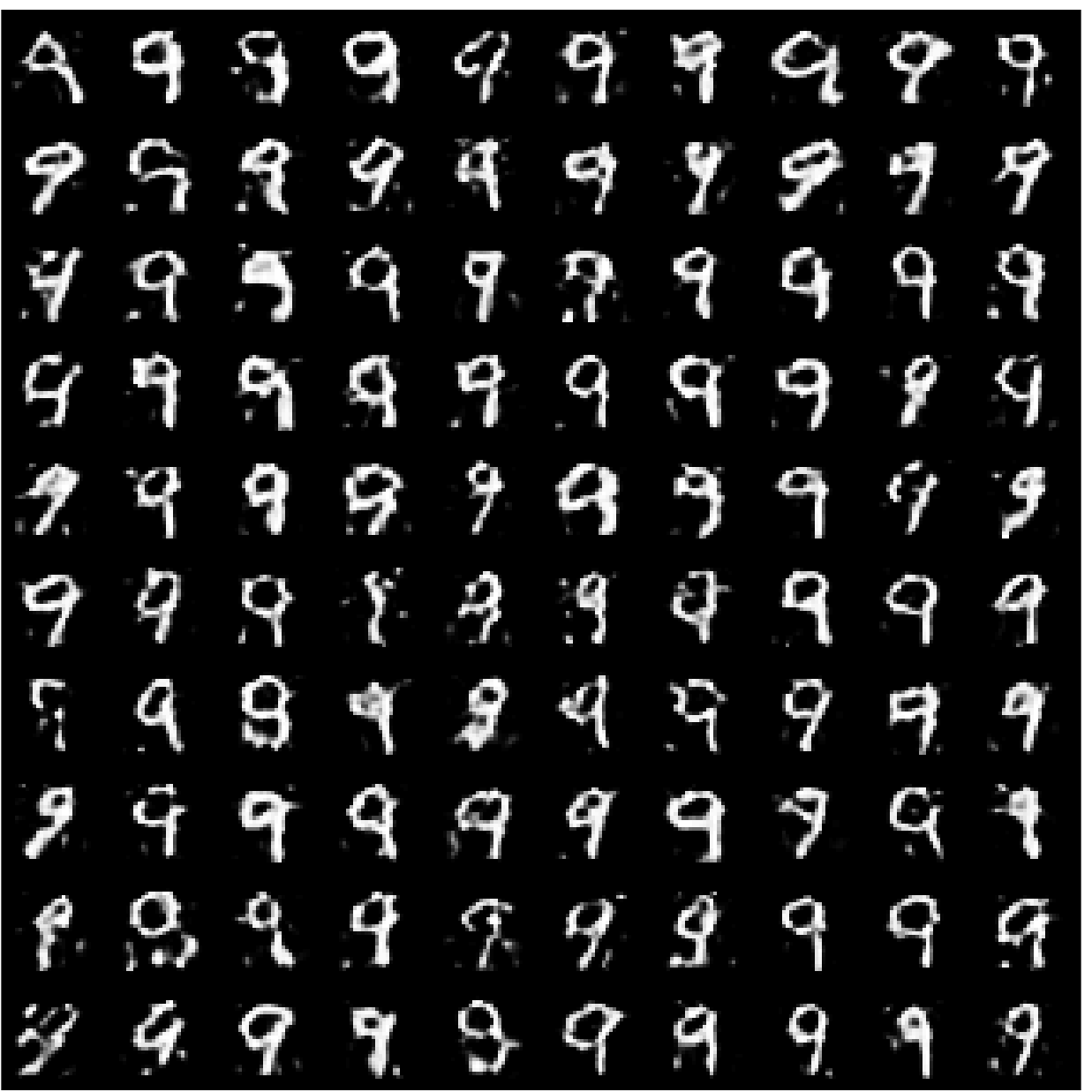}}
\caption{The denoising demonstration of the trained MVRBM over the images of digit 9.}
\label{Figure1Gao}
\end{figure}

In another attempt, we train a MVRBM with $N=20,000$ training samples ($2,000$ samples for each digit) and $T=3,000$ epoches, but the hidden size is set to 25. Based on the bilinear model of the MVRBM, the trained model parameters $U$ and $V$ can be jointly used as filters or feature extractor in terms of dictionary $U^T\otimes V^T$. We show some filters in Figure~\ref{Figure2Gao}(a), from which we can see that the learned filters are quite close to the Haar filters used in image processing.

Then we test the capacity of the learned MVRBM in dimensionality reduction and reconstruction. Figure~\ref{Figure2Gao}(b) shows several examples of the original and its corresponding reconstruction from low dimension representation. Average reconstruction error is 10.8488/(28*28).
\begin{figure}
\centering
\subfigure[The Trained Filters]{\includegraphics[width=0.15\textwidth]{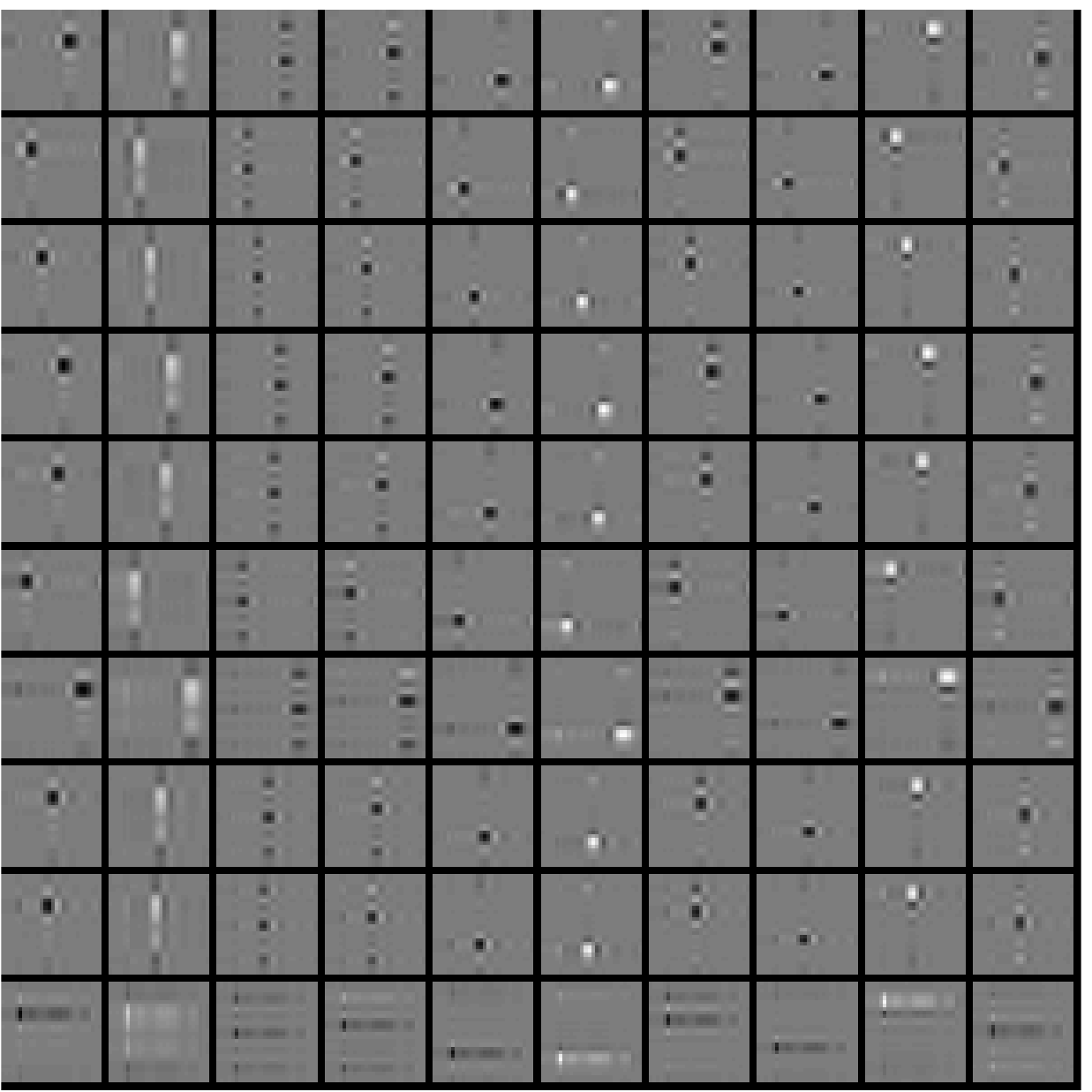}}\;
\subfigure[The Original Images]{\includegraphics[width=0.15\textwidth]{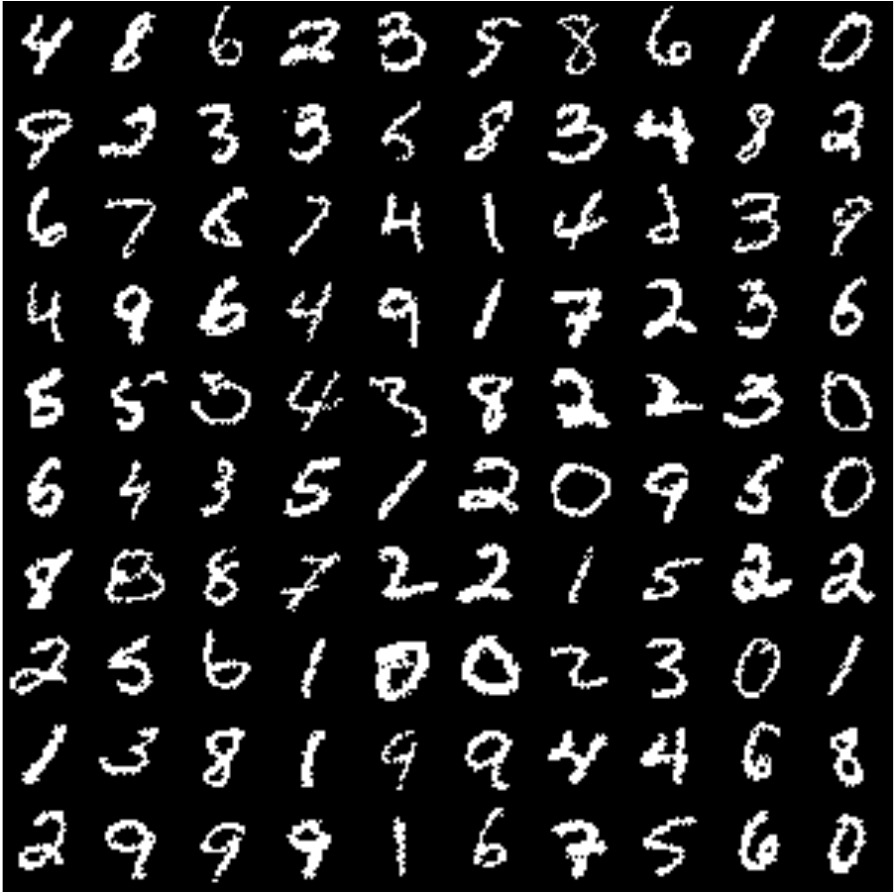}}\;
\subfigure[The Corresponding Reconstruction]{\includegraphics[width=0.15\textwidth]{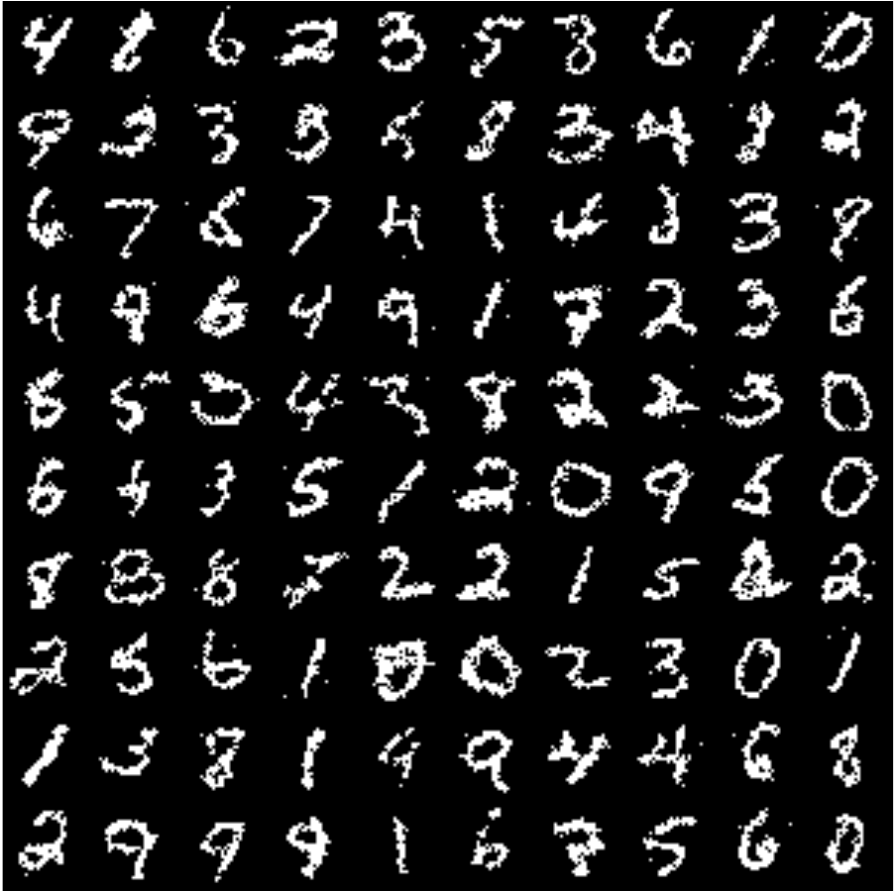}}
\caption{The learned MVRBM filters and the reconstruction demonstration of the trained MVRBM.}
\label{Figure2Gao}
\end{figure}

\subsection{Experiment 2: Handwritten Digit Classification}

The dataset used in Experiment 1 has been widely used for testing and evaluating classification or clustering algorithms~\cite{ChazalTapsonSchaik2015,Deng2012,KangGonugondlaMin-SunShanbhag2015}. In this experiment, we use this dataset to evaluate how well the proposed MVRBM is in feature extraction. In fact, the states over hidden layer can be regarded as new features of observed data. These new features will be piped into a process to train a classifier. As most existing classifiers are designed for vectorial data, in this experiment, the hidden matrix features given by MVRBM will be concatenated into vectors and then use the $K$ Nearest Neighbor ($K$-NN) Classifier with $K=1$ for classification.

We assess how different model training settings impact the 1-NN classifier performance. Under different training settings, we first train an MVRBM, use it as a feature extractor and then conduct 1-NN classification over all the testing digits. Finally we report the classification error rates.

In the first test, we fix the size of hidden units at $25\times25$ and iteration number for $T=2,000$ while varying the number of training samples from $100$ to $20,000$. We show the results of classification errors in Figure~\ref{Figure12}(a). Sufficient training samples lead to better classification performance.   In another test, we randomly choose $10,000$ training samples while varying the iteration number from $10$ to $3,000$ for training. The curve in Figure~\ref{Figure12}(b) shows the change of classification errors over the iteration numbers. We can observe that learning MVRBM model has been stable after 70 iterations. Particularly for the iteration numbers from $300$ to $3,000$, the classification error rate only changes from $0.0571$ to $0.0520$. Based on these observations, we recommend $N=20,000$ and $T=3,000$ in the most of the following experiments while comparing with other models.  In our experiment, we also note that the MVRBM with $50,000$ training samples gives a quite good accuracy of $0.0359$ although this accuracy goes up to $0.1387$ with only 600 training samples.

\begin{figure}
 \centering
 \subfigure[fixed $T=2000$] {\includegraphics[width=0.23\textwidth]{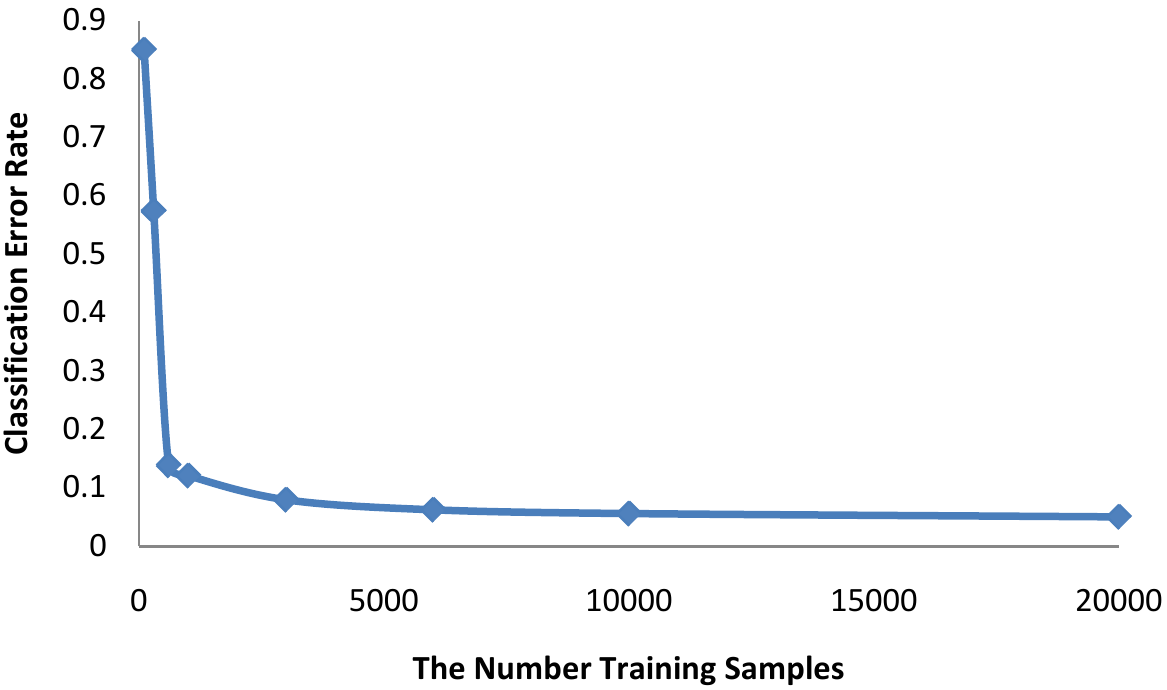}}\;
 \subfigure[fixed $N=10000$]{\includegraphics[width=0.23\textwidth]{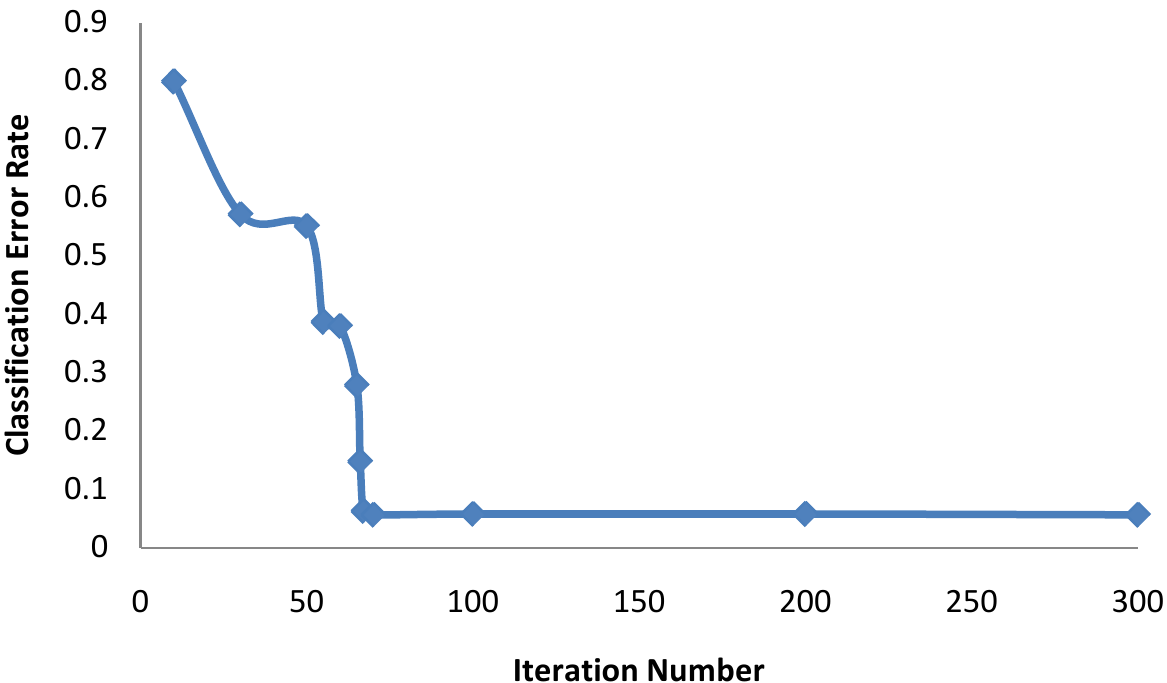}}
 \caption{Classification Errors vs $N$ (a) and $T$ (b).}
 \label{Figure12}
\end{figure}

Finally we compare the performance of the proposed model against other state-of-the-art methods which include the Deep Neural Network based drop-out method (DNN)~\cite{SrivastavaHintonKrizhevskySutskeverSalakhutdinov2014}, the Deep Belief Networks (DBN)~\cite{HintonS.Teh2006}, the Convolutional Neural Networks (CNN)~\cite{LeCunKavukcuogluFarabet2010} and the Sparse Autoencoder (SAE)~\cite{Andrew2011}.
These compared methods are implemented in the DeepLearn Toolbox which is available online at \url{https://github.com/rasmusbergpalm/DeepLearnToolbox} and we use their default parameter settings.
Figure~\ref{figure11}(a) and (b) show the overall results for which we fix the iteration number $T=3,000$ as suggested for MVRBM. The observation we can make from two figures is that with the sufficient training samples the newly proposed MVRBM is comparable to all the other methods while the new model outperforms others in the cases of few samples (less than 10,000). We believe this is due to the fact that the MVRBM has much less model parameters than other models, thus it is much more immune to overfitting.
 \begin{figure}
\centering
 \subfigure[Smaller $N$'s]{\includegraphics[width=0.23\textwidth]{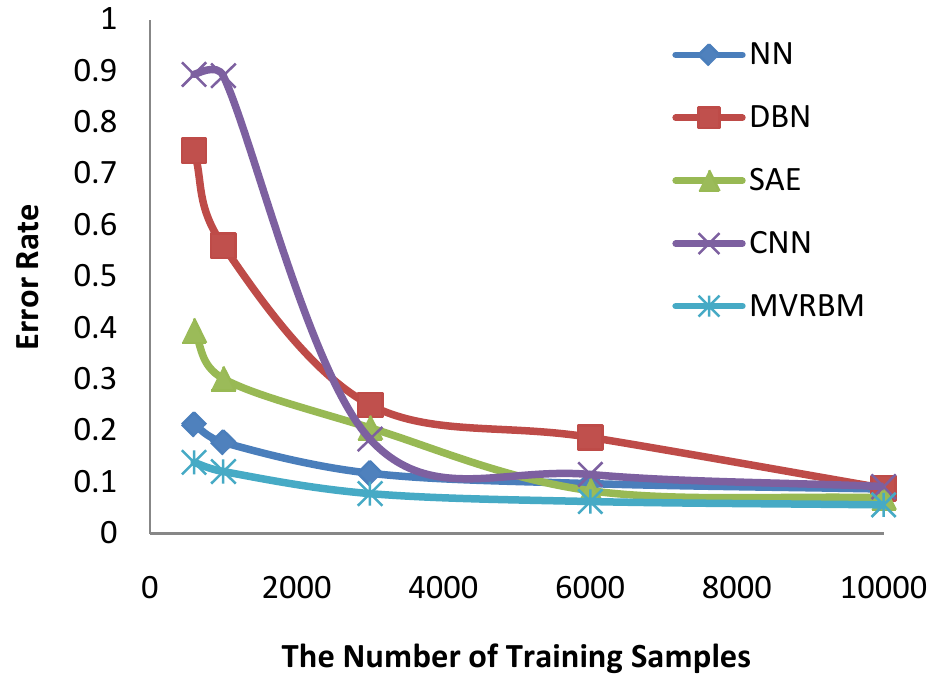}}
 \subfigure[Larger $N$'s]{\includegraphics[width=0.23\textwidth]{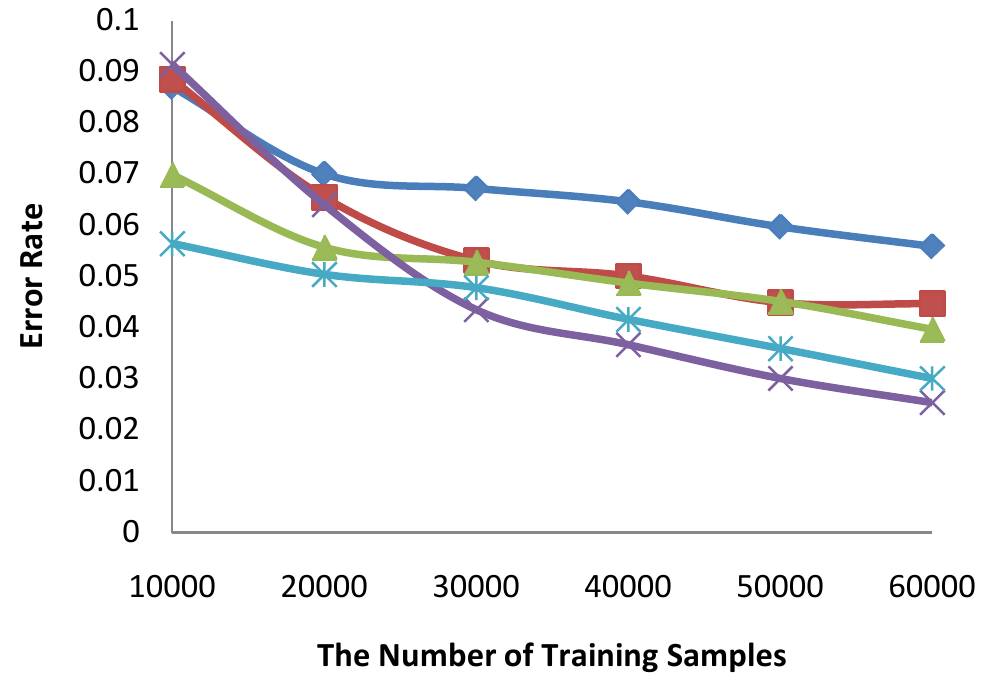}}
\caption{The Classification Errors vs $N$.}
\label{figure11}
\end{figure}

\subsection{Experiment 3: Image Super-resolution}

In this experiment we apply our Mutlimodal MVRBM model for image super-resolution. We follow the same setting used in~\cite{YangWrightHuangM2010} to prepare training data. The training patches are randomly taken from 69 Miscellaneous color images which are available online at \texttt{http://decsai.ugr.es/cvg/dbimagenes/}.

Each training sample consists of a high resolution patch $X$ (the raw image from luminance Y channel in the YCbCr color space) and four low resolution patches which are the directives of images in x-, y-, xx- and yy-directions, denoted by $Z^1$, $Z^2$, $Z^3$ and $Z^4$ of the Y channel. Hence at the visible layer we have five matrix patches $(X, Z^1, Z^2, Z^3,Z^4)$. The energy function defined in \eqref{EnergyMultiview} can be appropriately extended to cope with this case.

We select $N$ training samples from the Miscelaneous database, denoted by
$\mathcal{D}=\{ (X_1, Z^1_1, Z^2_1, Z^3_1,$ $Z^4_1),  \cdots, (X_N, Z^1_N, Z^2_N, Z^3_N, Z^4_N)\}$. In this experiment, we randomly sample $N=10,000$ training patches and try different image patch sizes of $10\times 10$, $15\times15$, $20\times20$, $30\times $ and $35\times 35$, all with a magnification factor of 2. The hidden size is fixed to 20 to demonstrate the potential of using overcomplete dictionary. We found this hidden size gives better results in all the tests conducted.
Figure~\ref{Figure3Gao}(a) shows sample patches $X$ of size 15 and Figure~\ref{Figure3Gao}(b) shows the learned $U$ and $V$ in terms of $U^T\otimes V^T$.

After training MMVRBM for each case, we use the following strategy to conduct super-resolution inference. Given a low resolution feature input $Z= (Z^1, Z^2, Z^3, Z^4)$, we first use any simple super-resolution algorithm
such as interpolation based methods to get an estimate $X^0$ of the desired super-resolution patch $X$. Then take as the input(s) both $Z$ and $X^0$ to the visible layer of a trained MMVRBM and run the MMVRBM training algorithm to transfer message from
the visible layer to the hidden layer to get the variable $Y$. Following that, the message $Y$ can be transferred back from the hidden layer to the visible layer and the super-resolution results can be taken from those $X$ units. In general, this gives a faster inference algorithm as demonstrated in our experiments. While necessary, this process can be run several cycles to reach equilibrium. The similar idea has been used in~\cite{YangWangLinCohenScottHuang2012} where a neural network is trained as a separate post-procedure.

We apply this inference for the super-resolution on Lena image in size $256\times 256$. Table~\ref{Table3Gao} shows the reconstruction errors for each case. Based on the results we recommend using patch size 30 for general super-resolution. In Figure~\ref{figure5lena}, we compare our method with several other methods on  Lena for super-resolution. In this experiment we fix the patch size to $15\times15$ and hidden size to $20\times20$. The size of low-resolution input image is $256\times256$. The PSNR of our method is $35.3006$dB, much higher than $34.1282$dB from the classic bicubic interpolation. As our model is in bilinear format which is a sub-model of the full linear model such as the \textit{Super-resolution via Sparse Representation} (SR) algorithm in~\cite{YangWrightHuangM2010}, the PSNR of our method is slightly inferior to the SR method, however the reconstruction time is much better than theirs\footnote{As the bicubic interpolation is implemented in Matlab and highly optimized, we did not report its super-resolution recovery time.}.
\begin{figure}
\centering
\subfigure[High Resolution Patches]{\includegraphics[width=0.20\textwidth]{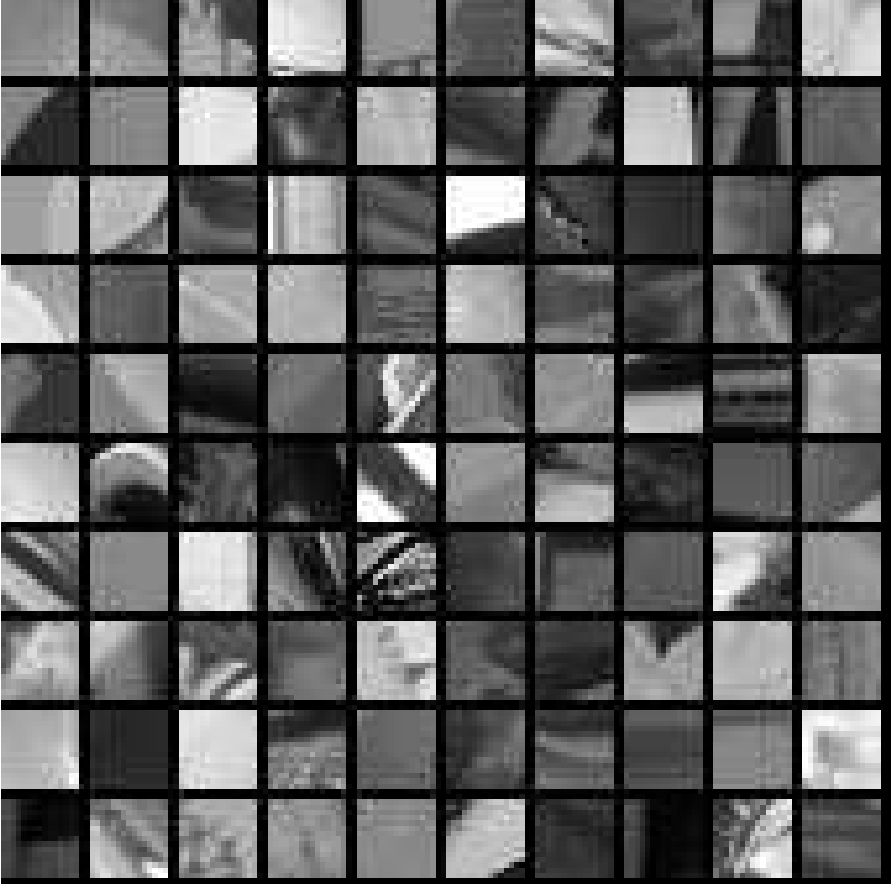}}\;
\subfigure[The learned $U^T\otimes V^T$]{\includegraphics[width=0.20\textwidth]{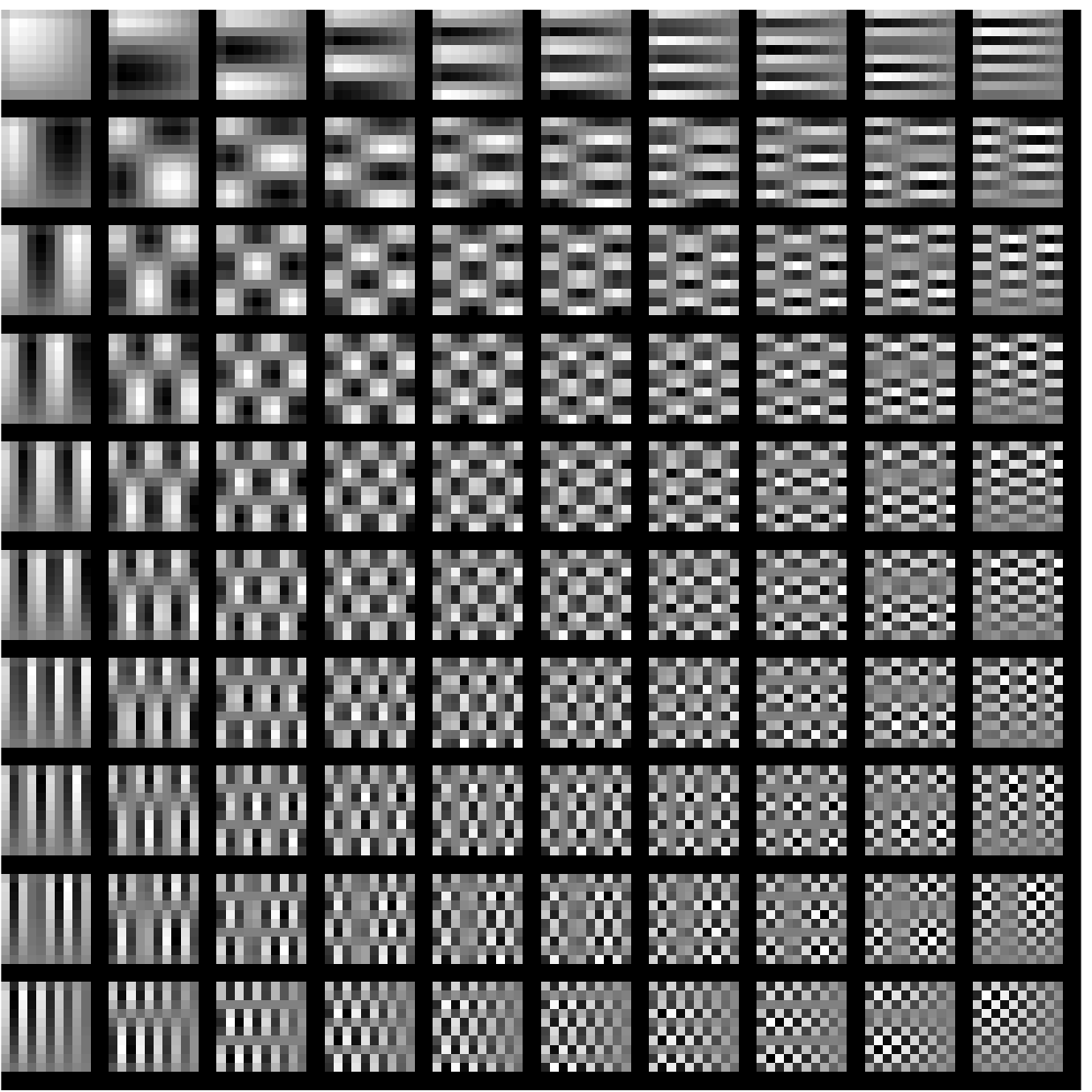}}
\caption{The selected training patches and the learned filters or dictionary in terms of $U^T\otimes V^T$. }
\label{Figure3Gao}
\end{figure}
\begin{table}
  \centering
  \begin{tabular}{|c|c|c|}
\hline
Patch Size& Hidden Size & PSNR(dB)\\[0.5ex]\hline
$10\times10$&$20\times20$&	35.1621\\\hline
$15\times15$&$20\times20$&	35.3227\\\hline
$20\times20$&$20\times20$&	35.3555\\\hline
$30\times30$&$20\times20$&	35.3606\\\hline
$35\times35$&$20\times20$&	35.3564\\\hline
\end{tabular}
  \caption{MVRBM models for different patches}\label{Table3Gao}
\end{table}

More tests have been conducted for natural image super-resolution, which are in the supplementary document to save the space of the paper. Here we only present one set of experiment as an example in Figure~\ref{natural}. The results show that the faster reconstruction of the proposed model against the SR algorithm while maintaining comparable reconstruction accuracy.

\begin{figure}
\centering%
\includegraphics[width=0.065\textwidth]{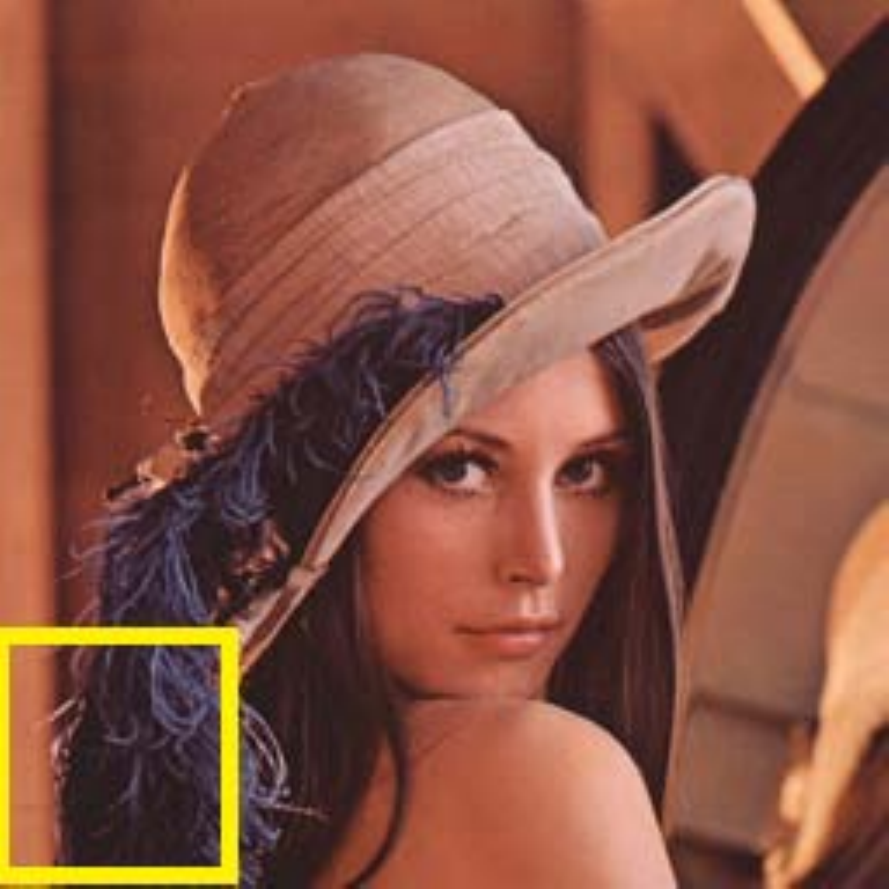}
\includegraphics[width=0.13\textwidth]{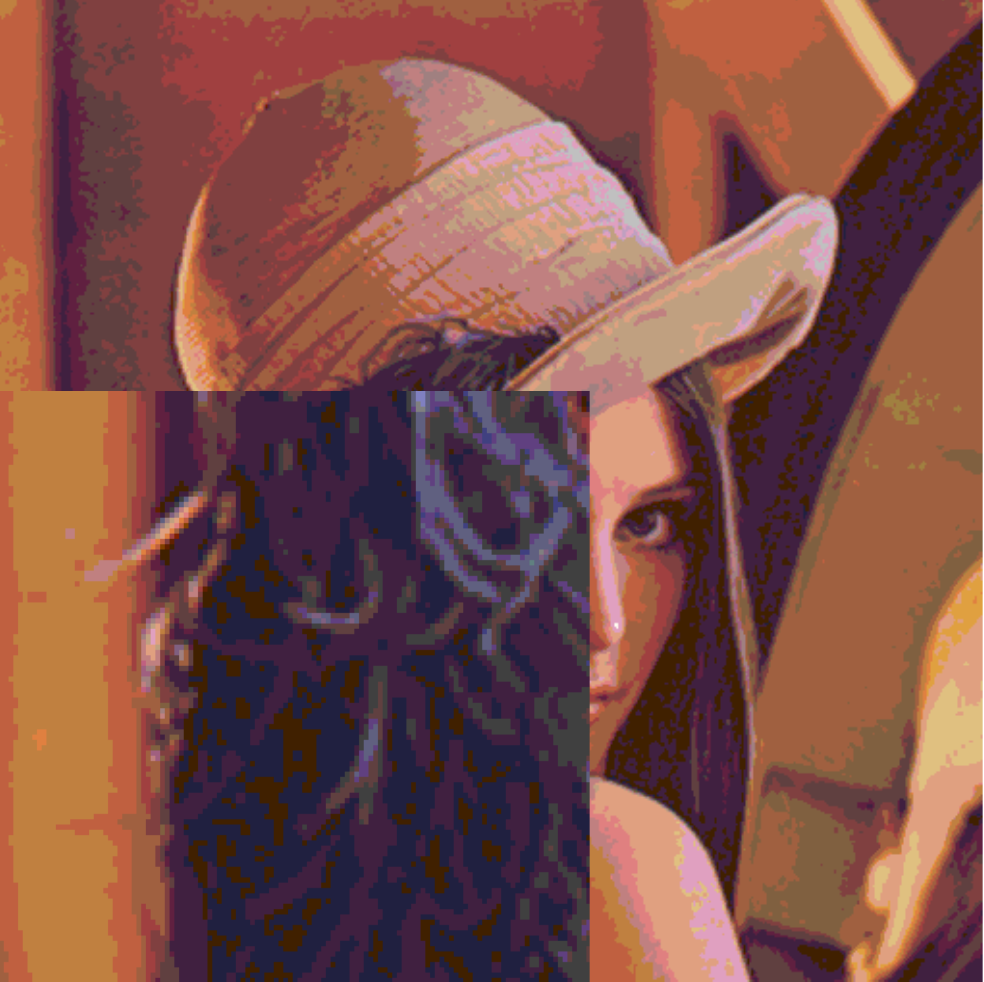}
\includegraphics[width=0.13\textwidth]{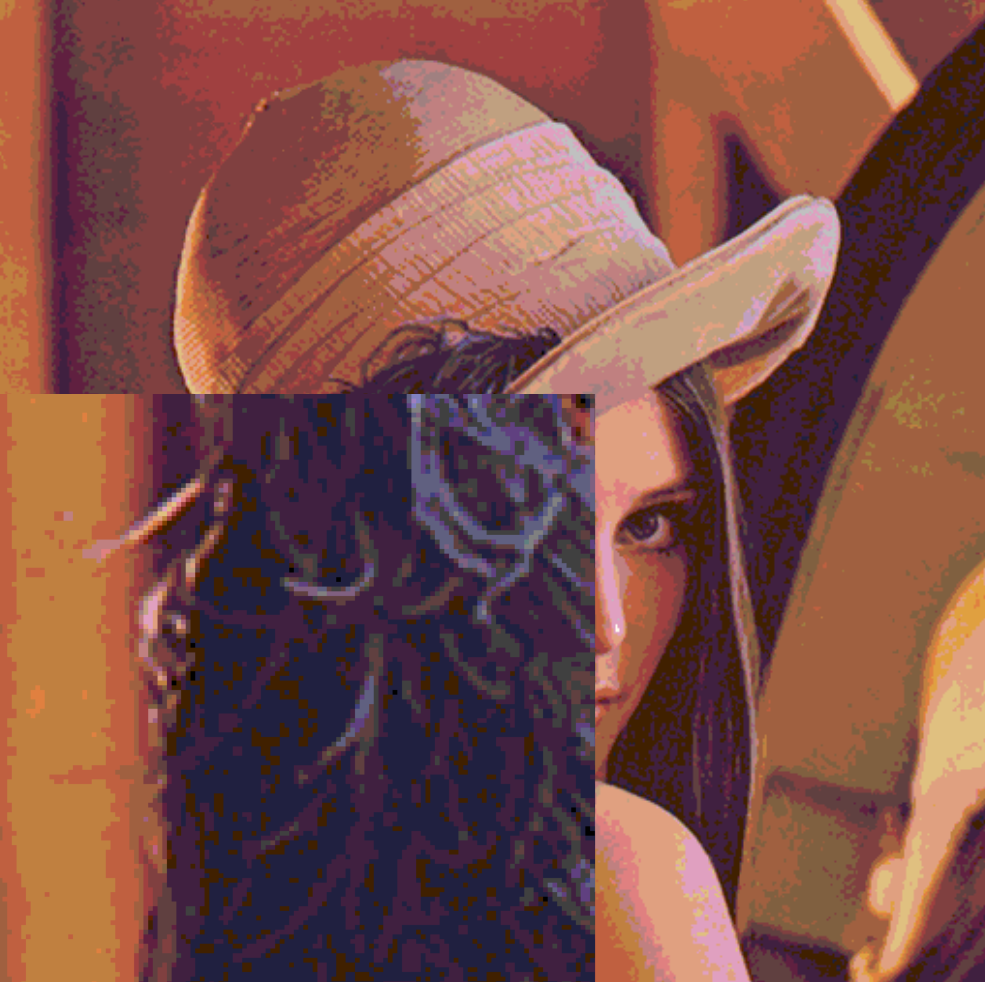}
\includegraphics[width=0.13\textwidth]{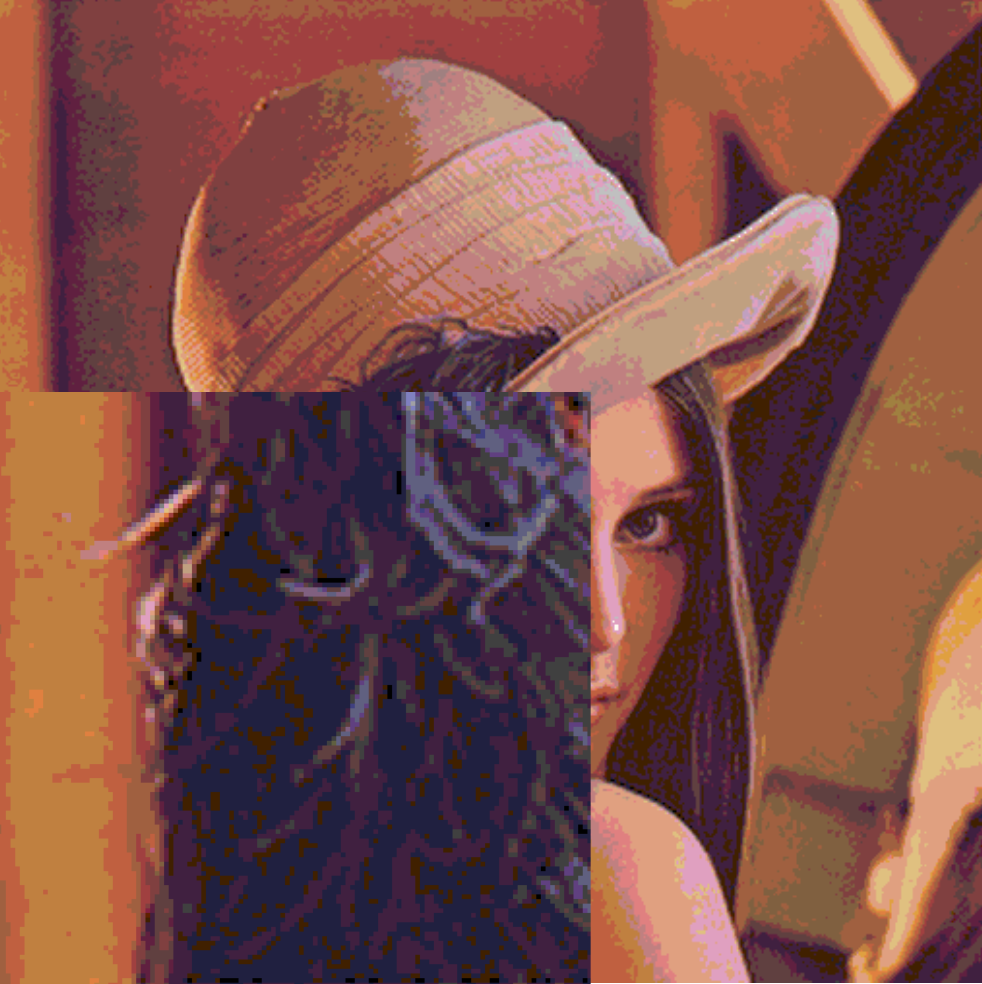}
\caption{Results of Lena image magnified by a factor of 2  and the corresponding RMSEs, Testing time and PSNR. Left to right: Input $(256\times256)$, Bicubic Interpolation (RMSE: 5.0134; PSNR: 34.1282dB), SR~\cite{YangWrightHuangM2010} (RMSE: 4.0900; Time: 679.529s; PSNR: 35.8963dB) and our Method (RMSE: 4.3804; Time: 36.116s; PSNR: 35.3006dB).}
\label{figure5lena}
\end{figure}

\begin{figure}
\centering%
\includegraphics[width=0.065\textwidth]{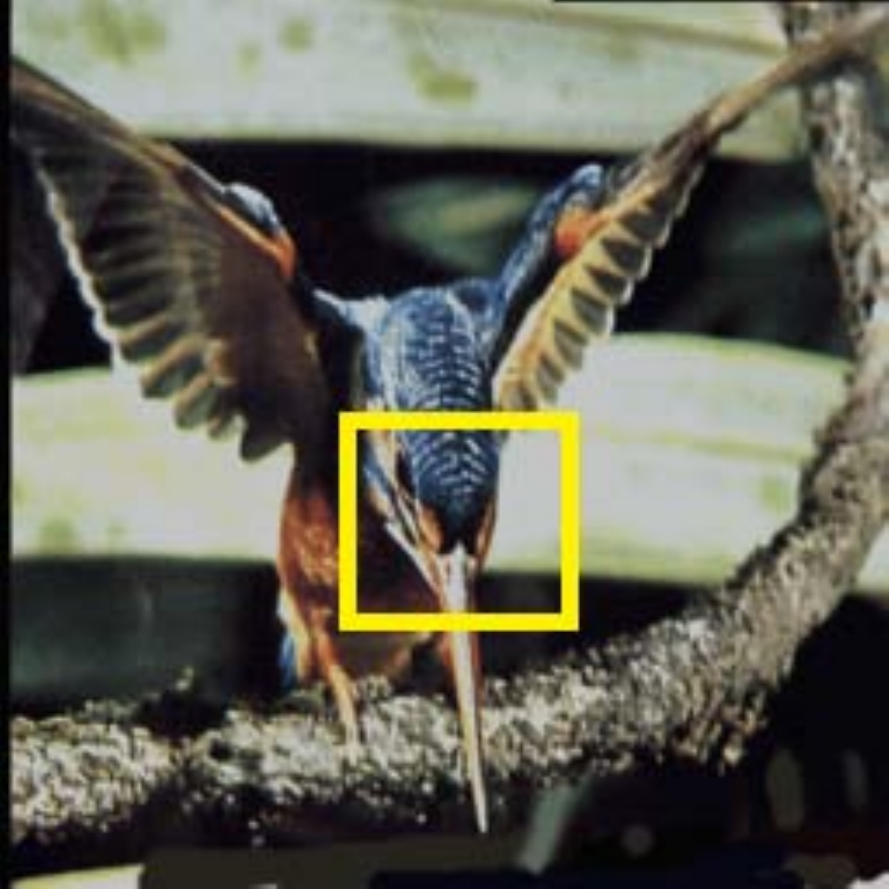}
\includegraphics[width=0.13\textwidth]{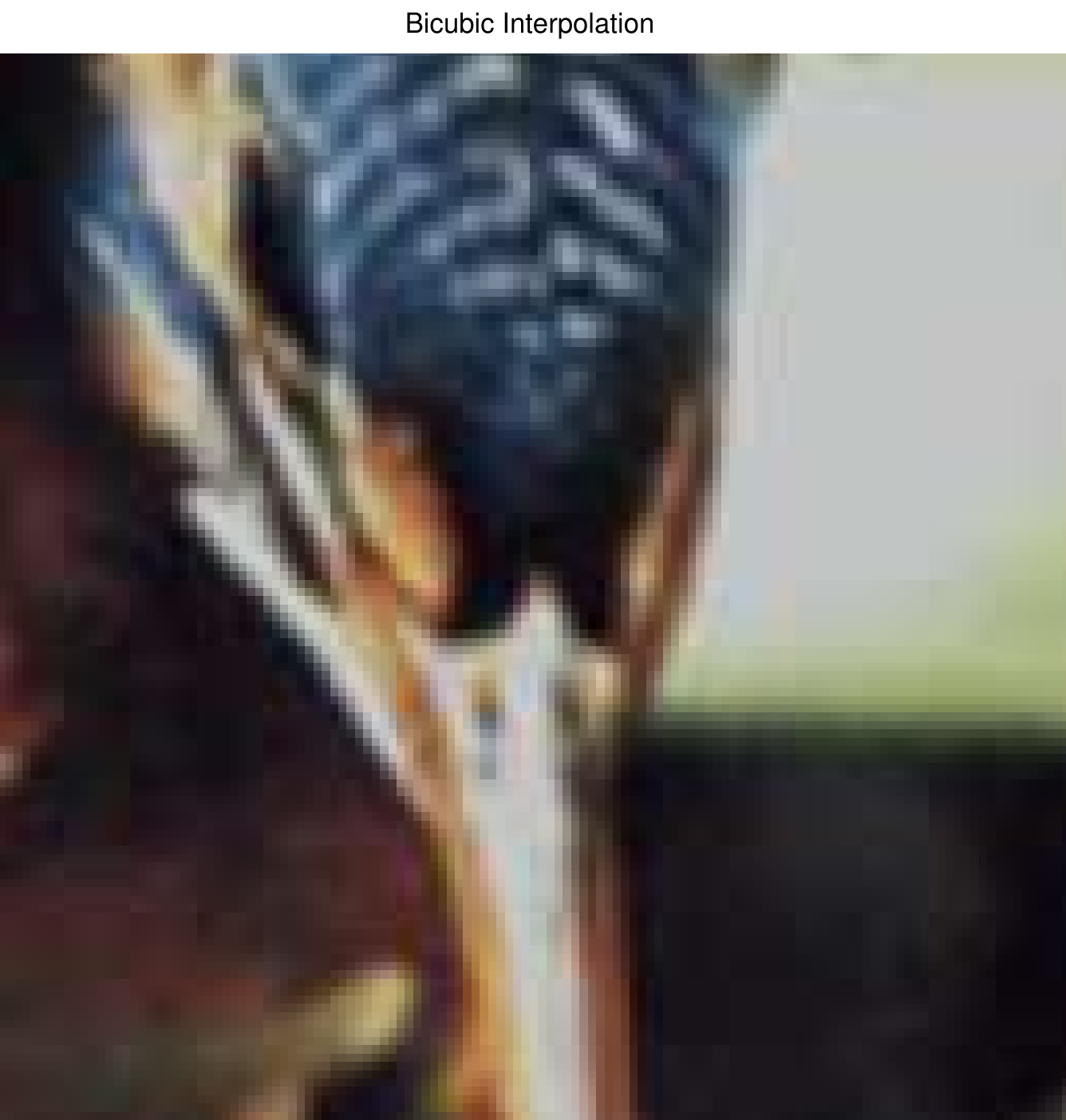}
\includegraphics[width=0.13\textwidth]{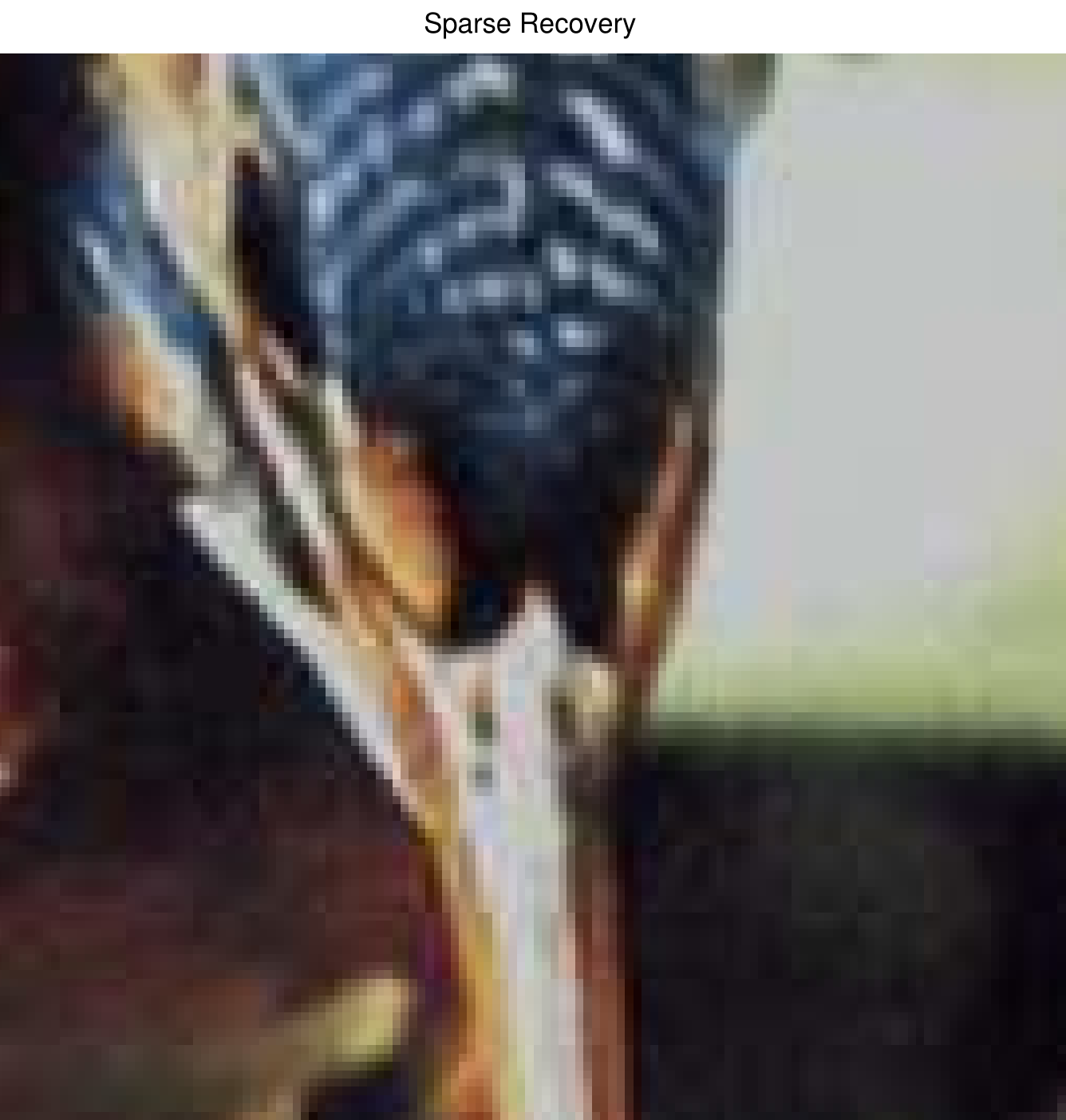}
\includegraphics[width=0.13\textwidth]{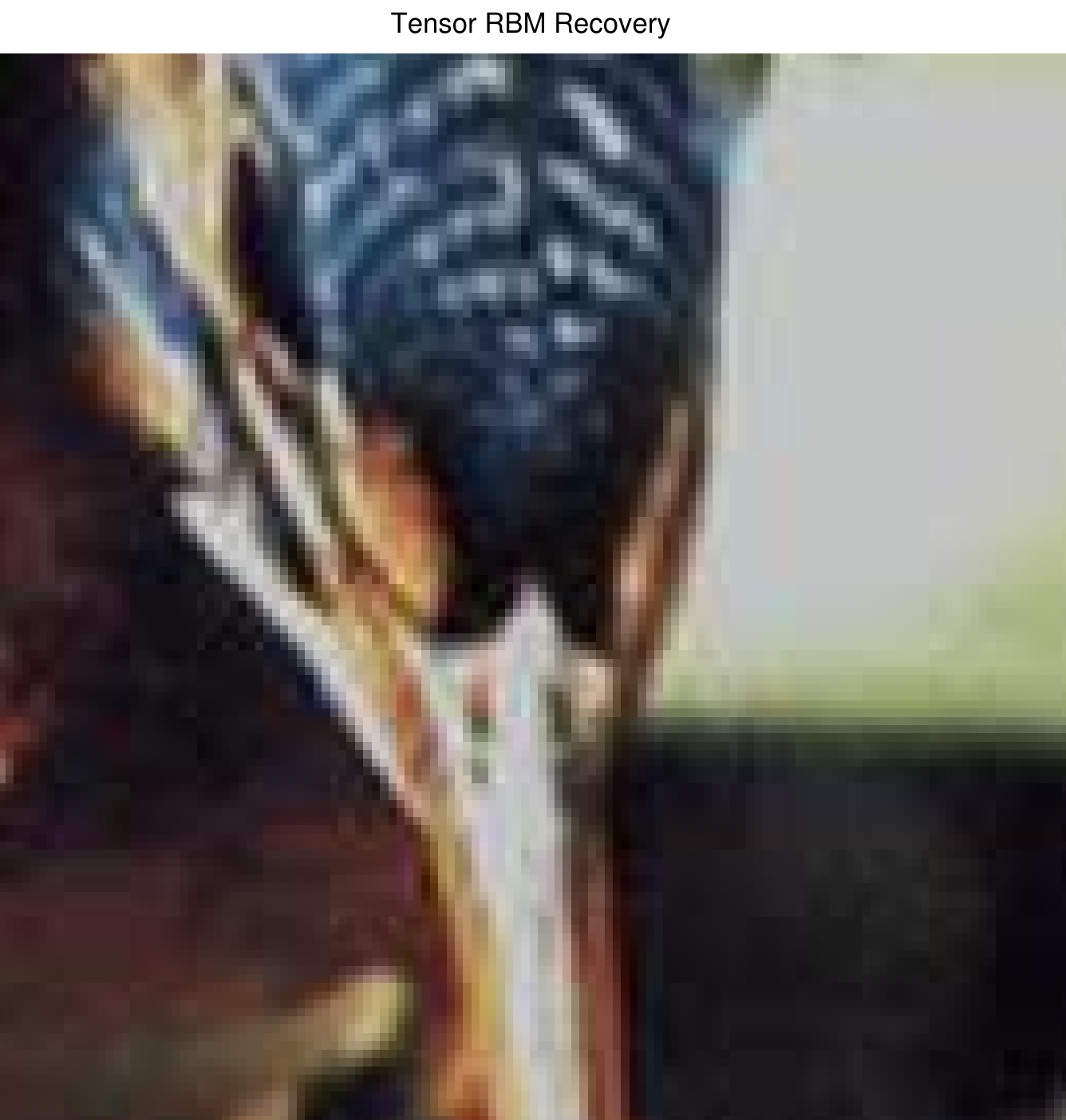}
\caption{Results of Bird image magnified by a factor of 2  and the corresponding RMSEs, Testing time and PSNR. Left to right: Input$(256\times256)$, Bicubic Interpolation (RMSE: 4.7932; PSNR: 34.5184 dB), SR~\cite{YangWrightHuangM2010} (RMSE: 3.7975; Time: 748.225s; PSNR: 36.5409dB) and our Method (RMSE: 3.9459; Time: 35.196s; PSNR: 36.2514dB).}
\label{natural}
\end{figure}

\section{Conclusions}\label{Sec:6}
In this paper, we proposed a novel model called Matrix Variate Restricted Boltzmann Machine (MVRBM) for 2D matrix variate data  by defining a bilinear connection between matrix variate visible layer and matrix variate hidden layer. Different from the traditional RBM which vectorizes the 2D matrix variate, this new model can make good use of spatial information in 2D matrix data, and be easily extended to any higher order tensor variate data.

In order to learn model parameters in MVRBM, we express the multiplicative interaction between visible and hidden units as a specified structure, and thus the number of free parameters in the model is significantly reduced, compared to the corresponding vectorized RBM models. The relevant learning algorithm for the new model has been investigated.

The experiments have demonstrated the new model is comparable to the classic RBM while maintaining good training and inferring computational complexity. This has been particularly demonstrated in the application to the image super-resolution problem.
Our model can be easily incorporated into a deeper structure. Using the deeper structure we may get more abstract feature and better performance. Our future work is to apply the method into the field of the deep learning and the construction of the deep neural network structures.

{\small

}

\end{document}